\documentclass[letterpaper,11pt]{article}

\usepackage{amsmath}
\usepackage{amsfonts}
\usepackage{amssymb}
\usepackage{amsthm}
\usepackage{url,ifthen}
\usepackage{enumitem}
\usepackage{srcltx}
\usepackage{multirow}
\usepackage{boxedminipage}
\usepackage[margin=1.1in]{geometry}
\usepackage{nicefrac}
\usepackage{xspace}
\usepackage{graphicx}
\usepackage{color}
\definecolor{DarkGreen}{rgb}{0.1,0.5,0.1}
\definecolor{DarkRed}{rgb}{0.5,0.1,0.1}
\definecolor{DarkBlue}{rgb}{0.1,0.1,0.5}

\usepackage[small]{caption}
\usepackage[pdftex]{hyperref}
\hypersetup{
    unicode=false,          
    pdftoolbar=true,        
    pdfmenubar=true,        
    pdffitwindow=false,      
    pdfnewwindow=true,      
    colorlinks=true,       
    linkcolor=DarkBlue,          
    citecolor=DarkGreen,        
    filecolor=DarkRed,      
    urlcolor=DarkBlue,          
    pdftitle={},
    pdfauthor={},
}

\def\draft{1}

\def\submit{0}

\ifnum\draft=1 
    \def\ShowAuthNotes{1}
\else
    \def\ShowAuthNotes{0}
\fi

\ifnum\submit=1
\newcommand{\forsubmit}[1]{#1}
\newcommand{\forreals}[1]{}
\else
\newcommand{\forreals}[1]{#1}
\newcommand{\forsubmit}[1]{}
\fi

\ifnum\ShowAuthNotes=1
\newcommand{\authnote}[2]{{ \footnotesize \bf{\color{DarkRed}[#1's Note:
{\color{DarkBlue}#2}]}}}
\else
\newcommand{\authnote}[2]{}
\fi

\newtheorem{theorem}{Theorem}[section]

\newtheorem{corollary}[theorem]{Corollary}

\newtheorem{claim}[theorem]{Claim}

\theoremstyle{definition}
\newtheorem{definition}[theorem]{Definition}

\newcommand{\chapterref}[1]{\hyperref[ch:#1]{Chapter~\ref{ch:#1}}}

\newcommand{\claimref}[1]{\hyperref[claim:#1]{Claim~\ref{claim:#1}}}

\newcommand{\corollaryref}[1]{\hyperref[cor:#1]{Corollary~\ref{cor:#1}}}

\newcommand{\definitionref}[1]{\hyperref[def:#1]{Definition~\ref{def:#1}}}
\newcommand{\equationlabel}[1]{\label{eq:#1}}
\newcommand{\equationref}[1]{\hyperref[eq:#1]{Equation~\ref{eq:#1}}}

\newcommand{\factref}[1]{\hyperref[fact:#1]{Fact~\ref{fact:#1}}}
\newcommand{\figurelabel}[1]{\label{fig:#1}}
\newcommand{\figureref}[1]{\hyperref[fig:#1]{Figure~\ref{fig:#1}}}
\newcommand{\tablelabel}[1]{\label{tab:#1}}
\newcommand{\tableref}[1]{\hyperref[tab:#1]{Table~\ref{tab:#1}}}

\newcommand{\itemref}[1]{\hyperref[item:#1]{Item~(\ref{item:#1})}}

\newcommand{\lemmaref}[1]{\hyperref[lem:#1]{Lemma~\ref{lem:#1}}}

\newcommand{\propref}[1]{\hyperref[prop:#1]{Proposition~\ref{prop:#1}}}

\newcommand{\propositionref}[1]{\hyperref[prop:#1]{Proposition~\ref{prop:#1}}}

\newcommand{\remarkref}[1]{\hyperref[rem:#1]{Remark~\ref{rem:#1}}}

\newcommand{\sectionref}[1]{\hyperref[sec:#1]{Section~\ref{sec:#1}}}
\newcommand{\theoremlabel}[1]{\label{thm:#1}}
\newcommand{\theoremref}[1]{\hyperref[thm:#1]{Theorem~\ref{thm:#1}}}

\usepackage[T1]{fontenc}
\usepackage{kpfonts}
\usepackage{microtype}

\newcommand{\Esymb}{\mathbb{E}}
\newcommand{\Psymb}{\mathbb{P}}

\DeclareMathOperator*{\E}{\Esymb}

\DeclareMathOperator*{\ProbOp}{\Psymb r}
\renewcommand{\Pr}{\ProbOp}

\newcommand{\mper}{\,.}
\newcommand{\mcom}{\,,}

\renewcommand{\hat}{\widehat}

\newcommand{\cA}{{\cal A}}

\newcommand{\cD}{{\cal D}}

\newcommand{\cT}{{\cal T}}

\newcommand{\defeq}{\stackrel{\small \mathrm{def}}{=}}

\renewcommand{\le}{\leqslant}

\renewcommand{\ge}{\geqslant}

\newcommand{\Set}[1]{\left\{#1\right\}}

\newcommand{\bits}{\{0,1\}}

\usepackage{bm}

\newcommand{\trans}{\top}

\newcommand{\ignore}[1]{}

\renewcommand{\epsilon}{\varepsilon}

\newcommand{\remove}[1]{}

\newenvironment{itm}
{\begin{itemize}[noitemsep,topsep=0pt,parsep=0pt,partopsep=0pt]}
{\end{itemize}}
\newenvironment{enum}
{\begin{enumerate}[noitemsep,topsep=0pt,parsep=0pt,partopsep=0pt]}
{\end{enumerate}}

\title{The Ladder: A Reliable Leaderboard for\\ Machine Learning Competitions}
\author{Avrim Blum\and Moritz Hardt}
\begin{document}
\maketitle
\begin{abstract}
The organizer of a machine learning competition faces the problem of
maintaining an accurate leaderboard that faithfully represents the quality of the
best submission of each competing team. What makes this estimation problem
particularly challenging is its sequential and adaptive nature. As
participants are allowed to repeatedly evaluate their submissions on the
leaderboard, they may begin to overfit to the holdout data that supports the
leaderboard. Few theoretical results give actionable advice on how to
design a reliable leaderboard. Existing approaches therefore often resort to
poorly understood heuristics such as limiting the bit precision of answers and
the rate of re-submission.

In this work, we introduce a notion of \emph{leaderboard accuracy} tailored to
the format of a competition. We introduce a natural algorithm called \emph{the Ladder} 
and demonstrate that it simultaneously supports strong theoretical guarantees in a 
fully adaptive model of estimation, withstands practical adversarial attacks, and 
achieves high utility on real submission files from an actual competition hosted by
Kaggle. 

Notably, we are able to sidestep a powerful recent hardness result for
adaptive risk estimation that rules out algorithms such as ours under a
seemingly very similar notion of accuracy. On a practical note, we provide a
completely parameter-free variant of our algorithm that can be
deployed in a real competition with no tuning required whatsoever. 
\end{abstract}
%

%
%

\section{Introduction}

Machine learning competitions have become an extremely popular format for
solving prediction and classification problems of all kinds. A number of
companies such as Netflix have organized major competitions in the past and some
start-ups like Kaggle specialize in hosting machine learning competitions.  In
a typical competition hundreds of participants will compete for prize money by
repeatedly submitting classifiers to the host in an attempt to improve on their
previously best score. The score reflects the performance of the classifier on
some subset of the data, which are typically partitioned into two sets: a
training set and a test set. The training set is publicly available with both
the individual instances and their corresponding class labels. The test set is 
publicly available as well, but the class labels are withheld. Predicting these 
missing class labels is the goal of the participant and a valid submission is simply 
a list of labels---one for each point in the test set.

The central component of any competition is the leaderboard which ranks all
teams in the competition by the score of their best submission. This leads to
the fundamental problem of maintaining a leaderboard that accurately reflects
the true strength of a classifier. What makes this problem so challenging is
that participants may begin to incorporate the feedback from the leaderboard into
the design of their classifier thus creating a dependence between the
classifier and the data on which it is evaluated. In such cases, it is well
known that the holdout set no longer gives an unbiased estimate of the
classifier's true performance. To counteract this problem, existing solutions
such as the one used by Kaggle further partition the test set into two parts.
One part of the test set is used for computing scores on the public
leaderboard. The other is used to rank all submissions after the competition
ended. This final ranking is often referred to as the \emph{private
leaderboard}. While this solution increases the quality of the private
leaderboard, it does not address the problem of maintaining accuracy on the
public leaderboard. Indeed, numerous posts on the forums of Kaggle report on
the problem of ``overfitting to the holdout'' meaning that some scores on the
public leaderboard are inflated compared to final scores. To mitigate this
problem Kaggle primarily restricts the rate of re-submission and to some extent 
the numerical precision of the released scores.

Yet, in spite of its obvious importance, there is relatively little theory on how to
design a leaderboard with rigorous quality guarantees. Basic questions remain
difficult to assess, such as, can we a priori quantify how accurate existing
leaderboard mechanisms are and can we design better methods? 

While the theory of estimating the true loss of a classifier or set of classifiers
from a finite
sample is decades old, much of theory breaks down due to the \emph{sequential}
and \emph{adaptive} nature of the estimation problem that arises when
maintaining a leaderboard. First of all, there is no a priori understanding of
which learning algorithms are going to be used, the complexity of the
classifiers they are producing, and how many submissions there
are going to be. Indeed, submissions are just a list of labels and do not even
specify how these labels were obtained. Second, any submission might
incorporate statistical information about the withheld class labels that was
revealed by the score of previous submissions. In such cases, the public
leaderboard may no longer provide an unbiased estimate of the true score.
To make matters worse, very recent results suggest that maintaining accurate
estimates on a sequence of many adaptively chosen classifiers may be
computationally intractable~\cite{HardtU14,SteinkeU14}.

\subsection{Our Contributions}

We introduce a notion of accuracy called \emph{leaderboard accuracy} tailored
to the format of a competition.  Intuitively, high leaderboard accuracy
entails that each score represented on the leaderboard is close to the true
score of the corresponding classifier on the unknown distribution from which
the data were drawn. Our primary theoretical contributions are the following.

\begin{enumerate}
\item
We show that there is a simple and natural algorithm we call \emph{Ladder} that achieves 
high leaderboard accuracy in a fully adaptive model of estimation in which we place no
restrictions on the data analyst whatsoever. In fact, we don't even limit the
number of submissions an analyst can make. Formally, our worst-case upper bound shows
that if we normalize scores to be in $[0,1]$ the maximum error of our
algorithm on any estimate is never worse than $O((\log(k)/n)^{1/3})$ where $k$
is the number of submissions and $n$ is the size of the data used to compute
the leaderboard. In contrast, we observe that the error of the Kaggle mechanism
(and similar solutions) scales with the number of submissions as $\sqrt{k}$ so that 
our algorithm features an exponential improvement in $k$.
\item
We also prove an information-theoretic lower bound on the leaderboard accuracy 
demonstrating that no estimator can achieve error smaller than
$\Omega((\log(k)/n)^{1/2}).$
\end{enumerate}

Complementing our theoretical worst-case upper bound and lower bound, we make
a number of practical contributions:

\begin{enumerate}
\item
We provide a \emph{parameter-free} variant of our algorithm that can be
deployed in a real competition with no tuning required whatsoever. 
\item To demonstrate the strength of our parameter-free algorithm we conduct two
opposing experiments. The first is an adversarial---yet practical---attack on the
leaderboard that aims to create as much of a bias as possible with a given
number of submissions. We compare the performance of the Kaggle mechanism to
that of the Ladder mechanism under this attack. We observe that the accuracy
of the Kaggle mechanism diminishes rapidly with the number of submissions,
while our algorithm encounters only a small bias in its estimates.
\item In a second experiment, we evaluate our algorithm on real submission
files from a Kaggle competition. The data set presents a difficult benchmark
as little overfitting occurred and the errors of the Kaggle leaderboard were
generally within the expected statistical deviations given the properties of the
data set. Even on this benchmark our algorithm produced a
leaderboard that is very close to that computed by Kaggle. Through a sequence
of significance tests we assess that the differences between the two
leaderboards on this competition are not statistically significant.
\end{enumerate}

In summary, our algorithm supports strong theoretical results while suggesting
a simple and practical solution. Importantly, it is one and the same
parameter-free algorithm that withstands our adversarial attack and
simultaneously achieves high utility in a real Kaggle competition. 

An important aspect of our algorithm is that it only releases a score to the
participant if the score presents a statistically significant improvement over
the previously best submission of the participant. Intuitively, this prevents
the participant from exploiting or overfitting to minor fluctuations in the
observed score values.

\subsection{Related Work}
There is a vast literature on preventing overfitting in the context of model
assessment and selection.  See, for example, Chapter 7 of~\cite{HastieTF01}
for background.  Two particularly popular practical approaches are various
forms of cross-validation and bootstrapping. It is important to note though that when
scoring a submission for the leaderboard, neither of these techniques applies.
One problem is that participants submit only a list of labels and not the
corresponding learning algorithms. In particular, the organizer of the competition has
no means of retraining the model on a different split of the data. Similarly,
the natural bootstrap estimate of the expected loss of a classifier given a
finite sample is simply the empirical average of the loss on the finite sample,
which is what existing solutions release anyway. The other substantial
obstacle is that even if these methods applied, their theoretical guarantees
in the adaptive setting of estimation are largely not understood.

A highly relevant recent work~\cite{DworkFHPRR14}, that inspired us,
studies a more general question: Given a sequence of \emph{adaptively chosen} bounded functions
$f_1,\dots,f_k\colon X\to\bits$ over a domain~$X,$ estimate
the expectations of these function~$\E f_1,\dots,\E f_k$ over an unknown distribution
$\cD,$ given $n$ samples from this distribution. If we think of each function
as expressing the loss of one classifier submitted to the leaderboard, then
such an algorithm could in principle be used in our setting. The main result
of~\cite{DworkFHPRR14} is an algorithm that achieves maximum error
\[
O\left(\min\left\{\log(k)^{3/7}(\log|X|)^{1/7}/n^{2/7},(\log|X|\log(k)/n)^{1/4}\right\}\right)\mper
\]
This bound readily implies
a corresponding result for leaderboard accuracy albeit worse than the one we show. 
One issue is that this algorithm requires the 
entire test set to be withheld and not just the labels as is required in the Kaggle
application. The bigger obstacle is that the algorithm is unfortunately not
computationally efficient and this is inherent.
In fact, no computationally efficient algorithm can give non-trivial error on
$k > n^{2+o(1)}$ adaptively chosen functions as was shown
recently~\cite{HardtU14,SteinkeU14} under a standard computational hardness assumption. 

Matching this hardness result, there is a computationally efficient algorithm
in~\cite{DworkFHPRR14} that achieves an error bound of
$O(k^{1/5}\log(k)^{3/5}/n^{2/5})$ which implies a bound on leaderboard
accuracy that is worse than ours for all $k>n^{1/3}.$ They also give an
algorithm (called EffectiveRounds) with accuracy $O(\sqrt{r\log(k)/n})$ when the 
number of ``rounds of adaptivity'' is at most~$r.$ While we do not have a
bound on $r$ in our setting better than $k$\footnote{The parameter $r$
corresponds to
the depth of the adaptive tree we define in the proof of \theoremref{ub}.
While we bound the size of the tree, the depth could be as large as $k.$}, the proof technique relies on sample 
splitting and a similar argument could be used to prove our upper bound. However, our 
argument does not require sample splitting and this is very important for the
practical applicability of the algorithm.

We sidestep the hardness result by going to
a more specialized notion of accuracy that is surprisingly still sufficient
for the leaderboard application. However, it does not resolve the more general
question raised in~\cite{DworkFHPRR14}. In particular, we do not always provide a
loss estimate for each submitted classifier, but only for those that made a
significant improvement over the previous best. This seemingly innocuous
change is enough to circumvent the aforementioned hardness results.

\subsection*{Acknowledgments}
We thank Ben Hamner at Kaggle Inc., for providing us with the submission files
from the Photo Quality competition, as well as many helpful discussions.  We
are grateful to Aaron Roth for pointing out an argument similar to that
appearing in the proof of Theorem~\ref{thm:ub} in a different context. We
thank John Duchi for many stimulating discussions. 

\subsection{Preliminaries}
Let $X$ be a data domain and $Y$ be a finite set of class labels, e.g.,
$X=\mathbb{R}^d$ and $Y=\{0,1\}.$ 
Rather than speaking of the score of a classifier we will use the term
\emph{loss} with the understanding that smaller is better.  A loss function is
a mapping of the form $\ell\colon Y\times Y\to[0,1]$ and a classifier is a
mapping $f\colon X \to Y.$ 
A standard loss function is the $0/1$-loss defined as $\ell_{01}(y,y')=1$ if $y\ne y'$ 
and $0$ otherwise.

We assume that we are given a sample $S = \{(x_1,y_1),\dots,(x_n,y_n)\}$ drawn i.i.d. from an
unknown distribution $\cD$ over $X\times Y.$
We define the \emph{empirical loss} of a classifier $f$ on the sample $S$ as
\[
R_S(f) \defeq \frac1n \sum_{i=1}^n \ell(f(x_i),y_i)\mper
\]
The \emph{true loss} is defined as
\[
R_\cD(f) \defeq \E_{(x,y)\sim\cD}\left[\ell(f(x),y))\right]\mper
\]
Throughout this paper we assume that $S$ consists of $n$
i.i.d. draws from~$\cD$ and $\ell$ is a loss function with bounded range.

\section{Sequential and Adaptive Loss Estimation}
In this section we formally define the adaptive model of estimation that we
work in and present our definition of \emph{leaderboard accuracy}.  Given a
sequence of classifiers $f_1,\dots,f_k$ and a finite sample~$S$ of size $n,$ a
fundamental estimation problem is to compute estimates $R_1,\dots,R_k$ such
that 
\begin{equation}\equationlabel{acc}
\Pr\Set{\exists t\in[k]\colon \left|R_t - R_\cD(f_t) \right|>\epsilon}\le\delta\mper
\end{equation}
The standard way of estimating the true loss is via the empirical loss. If we
assume that all functions $f_1,\dots,f_k$ are fixed independently of the
sample $S$, then Hoeffding's bound and the union bound imply
\begin{equation}\equationlabel{hoeffding}
\Pr\Set{\exists t \in [k]\colon \left|R_S(f_t) - R_\cD(f_t) \right|>\epsilon}\le2k\exp(-2\epsilon^2 n)\mper
\end{equation}
In the \emph{adaptive} setting, however, we assume that the classifier $f_t$ may be
chosen as a function of
the previous estimates and the previously chosen classifiers. Formally, there
exists a mapping $\cA$ such that for all $t\in[k]:$
\[
f_t = \cA(f_1,R_1,\dots,f_{t-1},R_{t-1})\mper
\]
We will assume for simplicity that $\cA$ is a
deterministic algorithm. The tuple $(f_1,R_1,\dots,f_{t-1},R_{t-1})$ is nevertheless
a random variable due to the random sample used to compute the estimates.

Unfortunately, in the case where the choice of $f_t$ depends on previous
estimates, we may no longer apply Hoeffding's bound to control $R_S(f_t)$.
In fact, recent work~\cite{HardtU14,SteinkeU14} shows that no computationally efficient
estimator can achieve error $o(1)$ on more than $n^{2+o(1)}$ adaptively chosen
functions (under a standard hardness assumption). Since we're primarily interested in
a computationally efficient algorithm, these hardness results demonstrate that
the goal of achieving the accuracy guarantee specified in
inequality~\eqref{eq:acc} is too stringent in the adaptive setting when $k$ is
large. We will therefore introduce a weaker notion of accuracy called
\emph{leaderboard accuracy} under which we can circumvent the hardness results
and nevertheless achieve a guarantee strong enough for our application.

\subsection{Leaderboard Accuracy}

The goal of an accurate leaderboard is to guarantee that at each step $t\le
k,$ the leaderboard accurately reflects the \emph{best} classifier among
those classifiers $f_1,\dots,f_k$ submitted so far. In other words, while we do not
need an accurate estimate for each $f_t,$ we wish to maintain that the $t$-th
estimate $R_t$ correctly reflects the minimum loss achieved by any classifier
so far. This leads to the following definition.

\begin{definition}
Given an adaptively chosen sequence of classifiers $f_1,\dots,f_k,$
we define the \emph{leaderboard error} of estimates $R_1,\dots,R_k$
as
\begin{equation}\equationlabel{lbacc}
\mathrm{lberr}(R_1,\dots,R_k)\defeq\max_{1\le t\le k}\left|\min_{1\le i\le t} R_\cD(f_i) - R_t\right|
\end{equation}
\end{definition}

Given an algorithm that achieves high leaderboard accuracy there are two
simple ways to extend it to provide a full leaderboard:
\begin{enumerate}
\item Use one instance of the algorithm for each team to maintain the best
score achieved by each team.
\item Use one instance of the algorithm for each rank on the leaderboard. When
a new submission comes in, evaluate it against each instance in descending order
to determine its place on the leaderboard.
\end{enumerate}
The first variant is straightforward to implement, but requires the assumption
that competitors don't use several accounts (a practice that is typically
against the terms of use of a competition). The second variant is more
conservative and does not need this assumption.

\section{The Ladder Mechanism}

We introduce an algorithm called the Ladder Mechanism that achieves
small leaderboard accuracy. The algorithm is very simple. For each given
function, it compares the empirical loss estimate of the function to the
previously smallest loss. If the estimate is below the previous best by some
margin, it releases the estimate and updates the best estimate. Importantly,
if the estimate is not smaller by a margin, the algorithm releases the
previous best loss (rather than the new estimate). A formal description
follows in \figureref{ladder}.

\begin{figure}[h]
\setlength{\fboxsep}{2mm}
\begin{center}
\begin{boxedminipage}{\textwidth}

\noindent {\bf Input:} Data set $S,$ step
size~$\eta>0$

\noindent {\bf Algorithm:}

\begin{itm}
\item Assign initial estimate $R_0\leftarrow \infty.$
\item {\bf For each} round $t \leftarrow 1,2 \ldots:$
\begin{enum}
\item Receive function $f_t\colon X\to Y$
\item {\bf If} $R_S(f_t) < R_{t-1} -\eta,$ assign $R_t\leftarrow
[R_S(f_t)]_\eta.$ {\bf Else} assign $R_t \leftarrow R_{t-1}.$
\item {\bf Output} $R_t$
\end{enum}
\end{itm}
\end{boxedminipage}
\end{center}
\vspace{-3mm}
\caption{\figurelabel{ladder} The Ladder Mechanism. We use the notation
$[x]_\eta$ to denote the number $x$ rounded to the nearest integer multiple of
$\eta$.}
\end{figure}

\begin{theorem}
\theoremlabel{ub}
For any sequence of adaptively chosen classifiers $f_1,\dots,f_k,$ the Ladder
Mechanism satisfies for all $t\le k$ and $\epsilon>0,$
\begin{equation}
\Pr\Set{ \left|\min_{1\le i\le t} R_\cD(f_i) - R_t\right| > \epsilon+\eta}
\le \exp\left(-2\epsilon^2n + (1/\eta+2)\log(4t/\eta)+1\right)\mper
\end{equation}
In particular, for some $\eta=O(n^{-1/3}\log^{1/3}(kn)),$ the Ladder Mechanism
achieves with high probability,
\[
\mathrm{lberr}(R_1,\dots,R_k)
\le O\left(\frac{\log^{1/3}(kn)}{n^{1/3}}\right)\mper
\]
\end{theorem}

\begin{proof}
Let $\cA$ be the adaptive analyst generating the function sequence.  Fix $t\le
k.$ The algorithm~$\cA$ naturally defines a rooted tree~$\cT$ of depth
$t$ recursively defined as follows:
\begin{enumerate}
\item The root is labeled by $f_1 = \cA(\emptyset).$
\item Each node at depth $1<i<t$ corresponds to one realization
$(h_1,r_1,\dots,h_{i-1},r_{i-1})$ of the random variable
$(f_1,R_1,\dots,f_{i-1},R_{i-1})$ and is labeled by
$h_i = \cA(h_1,r_1,\dots,h_{i-1},r_{i-1}).$ Its children are defined by each
possible value of the output $R_i$ of Ladder Mechanism on the sequence
$h_1,r_1,\dots,r_{i-1},h_i.$
\end{enumerate}
\begin{claim}
Let $B = (1/\eta+2)\log(4t/\eta).$ Then,
$|\cT|\le 2^B.$
\end{claim}
\begin{proof}
To prove the claim, we will uniquely encode each node in the tree using $B$
bits of information. The claim then follows directly. The compression argument
is as follows. We use $\lfloor\log(t)\rfloor\le \log(2t)$ bits to specify the
depth of the node in the tree. We then specify the index of each $1\le i\le t$
for which $R_i \le R_{i-1}-\eta$ together with the value $R_i.$ Note that
since $R_i \in[0,1]$ there can be at most $\lceil 1/\eta\rceil\le
(1/\eta)+1$ many such steps. Moreover, there are at most $\lceil 1/\eta
\rceil$ many possible values for $R_i=[R_S(f_i)]_\eta.$ Hence, specifying all
such indices requires at most $(1/\eta + 1)(\log(2/\eta)+\log(2t))$ bits.  It
is easy that this uniquely identifies each node in the graph, since for every
index $i$ not explicitly listed we know that $R_i=R_{i-1}.$ The total number
of bits we used is:
\[
(1/\eta+1)(\log(2/\eta)+\log(2t))+\log(2t)
\le (1/\eta +2)\log(4t/\eta) = B\mper
\]
\end{proof}
The theorem now follows by applying a union bound over all nodes in $\cT$ and
using Hoeffding's inequality for each fixed node. Let $F$ be the set of all
functions appearing in $\cT.$
\begin{align*}
\Pr\Set{\exists f\in F\colon \left|R_\cD(f) -R_S(f)\right|>\epsilon}
& \le 2|F|\exp(-2\epsilon^2n)\\
& \le 2^{B+1}\exp(-2\epsilon^2n)\le 2\exp(-2\epsilon^2n+B)\mper
\end{align*}
In particular,
\begin{align*}
\Pr\Set{\left|\min_{1\le i\le t}R_\cD(f_i) - \min_{1\le i \le t} R_S(f_i)\right|>\epsilon}
\le 2\exp(-2\epsilon^2n+B)\mper
\end{align*}
Moreover, it is clear that conditioned on the event that
\[
\left|\min_{1\le i\le t}R_\cD(f_i) - \min_{1\le i \le t} R_S(f_i)\right|\le
\epsilon,
\]
at step $i^*$ where the minimum of $R_\cD(f_i)$ is attained, the Ladder
Mechanism must output an estimate $R_{i^*}$ which is within $\epsilon+\eta$ of
$R_\cD(f_{i^*}).$ This concludes the proof.
\end{proof}

%
%

\subsection{A lower bound on leaderboard accuracy}

We next show that $\Omega(\sqrt{\log(k)/n})$ is a lower bound on the best possible
leaderboard accuracy that we might hope to achieve. This is true even if the
functions are not adaptively chosen but fixed ahead of time.

\begin{theorem}
There are classifiers $f_1,\dots f_k$ and a bounded loss function for which we
have the minimax lower bound
\[
\inf_{R}\sup_{\cD}
\E\left[\mathrm{lberr}(R(x_1,\dots,x_n))\right]
\ge \Omega\left(\sqrt{\frac{\log k}{n}}\right)\mper
\]
Here the infimum is taken over all estimators $R\colon X^n\to[0,1]^k$ that take $n$ samples
from a distribution $\cD$ and produce $k$ estimates $R_1,\dots,R_k=\hat\theta(x_1,\dots,x_n).$
The expectation is taken over $n$ samples from~$\cD.$
\end{theorem}

\begin{proof}
We will reduce the problem of mean estimation in a certain high-dimensional
distribution family to that of obtaining small leaderboard error.
Our lower bound then follows from lower bounds
for the corresponding mean estimation problem.

Let $X=\mathbb{R}^k$ and take
the functions $f_1,\dots,f_k$ to be the $k$ coordinate projections $f_i(x)=x_i,$ for $1\le i\le k.$
Let the loss function be the projection onto its first argument $\ell(y',y) = y'$ 
so that $\ell(f_i(x),y)=x_i.$
Consider the family of distributions $D_\epsilon = \{\cD_1,\dots,\cD_k\}$ where
$\cD_i$ is uniform over $\{0,1\}^k$ except that the $i$-th coordinate
satisfies $\E_{x\sim\cD_i} x_i = 1/2-\epsilon$ for some $\epsilon\in(0,1/4)$ that we
will determine later.  Now, we have
\[
R_{\cD_i}(f_j) = \begin{cases}
1/2 & i\ne j\\
1/2 - \epsilon & \text{o.w.}
\end{cases}
\]
Denote the mean of a distribution~$\cD$ by $\theta(\cD) = \E_{x\sim\cD}\left[x\right]$ and
note that $\theta(\cD) =  (R_{\cD}(f_1),\dots,R_{\cD}(f_k))^\trans.$
We claim that obtaining small leaderboard error on $f_1,\dots,f_k$ is at least
as hard estimating the means of an unknown distribution in $D_\epsilon.$
Formally,
\begin{equation}
\equationlabel{infs}
\inf_{R}\sup_{\cD\in D_\epsilon}
\E\left[\mathrm{lberr}(R(x_1,\dots,x_n))\right]
\ge \frac13\inf_{\hat\theta}\sup_{\cD\in D_\epsilon}
\E\left[\left\|\hat\theta(x_1,\dots,x_n)-\theta(\cD)\right\|_\infty\right]\mper
\end{equation}
Indeed, let $R$ be the estimator that achieves minimax leaderboard accuracy.
Define the estimator $\hat\theta$ as follows:
\begin{enumerate}
\item Given $x_1,\dots,x_n$ compute $R_1,\dots R_k = R(x_1,\dots,x_n).$
\item Let $i$ be the first coordinate in the sequence $R_1,\dots,R_k$
which is less than $1/2-\epsilon/2.$
\item Output the vector $\hat\theta(x_1,\dots,x_n)$ which is
$1/2-\epsilon$ in the $i$-th coordinate and $1/2$ everywhere else.
\end{enumerate}
Note that the $\ell_\infty$-error of $\hat\theta$ is always at most
$\epsilon,$ since all means parameters in the family $D_\epsilon$ are
$\epsilon$-close in $\ell_\infty$-norm.
Suppose then that $\mathrm{lberr}(R(x_1,\dots,x_n))\le\epsilon/3$ and suppose
that $\cD_i$ is the unknown distribution for some $i\in[k].$
In this case, we claim that $\|\theta(x_1,\dots,x_n)-\theta(\cD_i)\|_\infty=0.$
Indeed, the first coordinate for which $R(x_1,\dots,x_n)$ is less than
$1/2-\epsilon/2$ must be~$i.$ This follows from the definition of leaderboard
error and the assumption that $R$ had error $\epsilon/3$ on $x_1,\dots,x_n.$
This establishes inequality~\eqref{eq:infs}.

Finally, it is well known and follows from Fano's inequality that
for some $\epsilon =\Omega(\sqrt{\log(k)/n}),$
\[
\inf_{\hat\theta}\sup_{\cD\in D_\epsilon}
\E\left[\left\|\hat\theta(x_1,\dots,x_n)-\theta(\cD)\right\|_\infty\right]
\ge \epsilon\mper
\]
For completeness we include the argument. Let $V$ be a random index in $[k]$
and assume that $X$ is a random sample from $\cD_i$ conditional on $V=i.$
Note that the set $P=\{\theta(\cD_i)\}_{i\in[k]}$ forms an
$(\epsilon/2)$-packing in the $\ell_\infty$-norm.
Hence, by Fano's inequality (see~e.g.~\cite{Hasminskii78,Tsybakov08}),
\[
\inf_{\hat\theta}\sup_{\cD\in D_\epsilon}
\E\left[\right\|\hat\theta(x_1,\dots,x_n)-\theta(\cD)\left\|_\infty\right]
\ge \frac\epsilon4\left(1-\frac{I(V;X^n)+\log 2}{\log|P|}\right)\mcom
\]
where $I(V;X)$ is the mutual information between $V$ and $X.$
Moreover, it is known that
\[
I(V;X^n)\le\frac1{|P|^2}\sum_{i,j\in[k]}
D_{\mathrm{kl}}\left(\cD_i^n\left\|\cD_j^n\right.\right)
\le O(\epsilon^2n)\mper
\]
In the second inequality we used that the Kullback-Leibler divergence between
a Bernoulli random variable with bias $1/2$ and another one with bias
$1/2-\epsilon$ is at most $O(\epsilon^2)$ for all $0<\epsilon<1/4.$  Moreover,
the Kullback-Leibler divergence of $n$ independent samples is at most $n$
times the divergence of a single sample.
We conclude that
\[
\inf_{\hat\theta}\sup_{\cD\in D_\epsilon}
\E\left[\left\|\hat\theta(x_1,\dots,x_n)-\theta(\cD)\right\|_\infty\right]
\ge \frac\epsilon4\left(1-\frac{O(\epsilon^2n)+\log 2}{\log k}\right)\mper
\]
Setting $\epsilon = c\sqrt{\log(k)/n}$ for small enough constant $c>0$
completes the proof.
\end{proof}

\section{A parameter-free Ladder mechanism}

When applying the Ladder Mechanism in practice it can be difficult to choose a
fixed step size~$\eta$ ahead of time that will work throughout an entire
competition. We therefore now give a completely parameter-free version of our
algorithm that we will use in our experiments. The algorithm
adaptively finds a suitable step size based on previous submissions to the
algorithm. The idea is to perform a statistical significance test to judge
whether the given submission improves upon the previous one. The test is such
that as the best classifier gets increasingly accurate, the step size shrinks
accordingly.

The empirical loss of a classifier is the average of $n$ bounded numbers and
follows a very accurate normal approximation for sufficiently large $n$ so
long as the loss is not biased too much towards $0.$ In our setting, the
typical loss if bounded away form $0$ so that the normal approximation is
reasonable. In order to test whether the empirical loss of one classifier is
significantly below the empirical loss of another classifier, it is
appropriate to perform a one-sided \emph{paired} $t$-test. A paired test has
substantially more statistical power in settings where the loss vectors that
are being compared are highly correlated as is common in a competition. 

To recall the definition of the test, we denote the \emph{sample standard
deviation} of an $n$-dimensional vector vector $u$ as
$\mathrm{std}(u)=\sqrt{\frac1{n-1}\sum_{i=1}^n (u_i
-\mathrm{mean}(u))^2}\mcom$  where $\mathrm{mean}(u)$ denotes the average of
the entries in $u.$ With this notation, the paired $t$-test statistic given
two vectors $u$ and $v$ is defined as
\begin{equation}\equationlabel{ttest}
t = \sqrt{n}\cdot\frac{\mathrm{mean}(u-v)}{\mathrm{std}(u-v)}\mper
\end{equation}
Keeping this definition in mind, our parameter-free Ladder mechanism in \figureref{ladderpf} is now
very natural. On top of the loss estimate, it also maintains the loss vector
of the previously best classifier (starting with the trivial all zeros loss
vector).

\begin{figure}[h]
\setlength{\fboxsep}{2mm}
\begin{center}
\begin{boxedminipage}{\textwidth}

\noindent {\bf Input:} Data set $S=\{(x_1,y_1),\dots(x_n\dots,y_n)\}$ of size $n$

\noindent {\bf Algorithm:}

\begin{itm}
\item Assign initial estimate $R_0\leftarrow \infty,$ and loss vector
$\ell_0=(0)_{i=1}^n.$ 
\item {\bf For each} round $t \leftarrow 1,2 \ldots, k:$
\begin{enum}
\item Receive function $f_t\colon X\to Y.$
\item Compute loss vector $l_t \leftarrow (\ell(f_t(x_i),y_i))_{i=1}^n$
\item Compute the sample standard deviation $s\leftarrow \mathrm{std}(l_t-l_{t-1}).$
\item {\bf If} $R_S(f_t) < R_{t-1} - s/\sqrt{n}$
\begin{enum}
\item $R_t\leftarrow [R_S(f_t)]_{1/n}.$
\end{enum}
\item {\bf Else} assign $R_t \leftarrow R_{t-1}$ and $l_t \leftarrow l_{t-1}.$
\item {\bf Output} $R_t$
\end{enum}
\end{itm}
\end{boxedminipage}
\end{center}
\vspace{-3mm}
\caption{\figurelabel{ladderpf} The parameter-free Ladder Mechanism. We use the notation
$[x]_\eta$ to denote the number $x$ rounded to the nearest integer multiple of
$\eta$.}
\end{figure}

The algorithm in \figureref{ladderpf} releases the estimate of $R_S(f_t)$ up
to an error of $1/n$ which is significantly below the typical step size of
$\Omega(1/\sqrt{n}).$ Looking back at our analysis, this is not a problem
since such an estimate only reveals $\log(n)$ bits of information which is the
same up to constant factors as an estimate that is accurate to within
$1/\sqrt{n}.$ The more critical quantity is the step size as it controls how
often the algorithm releases a new estimate. 

In the following sections we will show that the parameter-free Ladder
mechanism achieves high accuracy both under a strong attack as well as on a
real Kaggle competition.

\subsection{Remark on the interpretation of the significance test}

For sufficiently large $n,$ the test statistic on the left hand side
of~\eqref{eq:ttest} is well approximated by a Student's $t$-distribution with
$n-1$ degrees of freedom.  The test performed in our algorithm at each step 
corresponds to refuting the null hypothesis roughly at the $0.15$ significance level. 

It is important to note, however, that our use of this significance test is
primarily heuristic. This is because for $t>1,$ due to the adaptive choices of
the analyst, the function $f_t$ may in general not be independent of the
sample~$S.$ In such a case, the Student approximation is no longer valid.
Besides we apply the test many times, but do not control for multiple
comparisons. Nevertheless, the significance test is an intuitive guide for
deciding which improvements are statistically significant.

\section{The boosting attack}

In this section we describe a new canonical attack that an adversarial analyst
might perform in order to boost their ranking on the public leaderboard.
Besides being practical in some cases, the attack also serves as an analytical
tool to assess the accuracy of concrete mechanisms.

For simplicity we describe the attack only for the $0/1$-loss although it
generalizes to other reasonable functions such as the clipped logarithmic loss
often used by Kaggle. We assume that the hidden solution is a vector
$y\in\{0,1\}^n.$ The analyst may submit a vector $u\in\{0,1\}^n$ and observe
(up to small enough error) the loss
\[
\ell_{01}(y,u)\defeq \frac1n\sum_{i=1}^n \ell_{01}(u_i,y_i)
\]
The attack proceeds as follows:
\begin{enumerate}
\item Pick $u_1,\dots,u_k\in\{0,1\}^n$ uniformly at random.
\item Observe loss estimates $l_1,\dots,l_k\in[0,1].$
\item Let $I = \Set{i \colon l_i \le 1/2}.$
\item Output $u^*=\mathrm{maj}\left(\Set{u_i\colon i\in I}\right),$ where the
majority function is applied coordinate-wise.
\end{enumerate}

The vector $y$ corresponds to the target set of labels used for the public
leaderboard which the analyst does not know. The vectors $u_1,\dots,u_k$
represent the labels given by a sequence of $k$ classifiers.

The next theorem follows from a standard ``boosting argument'' using
properties of the majority function and the fact that each $u_i$ for $i\in I$
has a somewhat larger than expected correlation with $y.$

\begin{theorem}\theoremlabel{boosting}
Assume that $|l_i-\ell_{01}(y,u_i)|\le n^{-1/2}$ for all $i\in[k].$
Then, the boosting attack finds a vector $u^*\in\bits^n$ so
that with probability $2/3,$
\[
\frac1n\sum_{i=1}^n \ell_{01}(u_i^*,y_i)
\le \frac12 - \Omega\left(\sqrt{\frac{k}{n}}\right)\mper
\]
\end{theorem}

The previous theorem in particular demonstrates that the Kaggle mechanism has
poor leaderboard accuracy if it is invoked with rounding parameter $\alpha\le
1/\sqrt{n}.$ The currently used rounding parameter is $10^{-5}$ which
satisfies this assumption for all $n\le 10^{10}.$

\begin{corollary}
There is a sequence of adaptively chosen classifiers $f_1,\dots,f_k$ such that
if $R_i$ denotes the minimum of the first $i$ loss estimates returned by the
Kaggle mechanism (as described in \figureref{kaggle}) with accuracy $\alpha \le
1/\sqrt{n}$ where $n$ is the size of the data set, then
with probability $2/3$ the estimates $R_1,\dots,R_k$ have leaderboard error
\[
\mathrm{lberr}(R_1,\dots,R_k)\ge \Omega\left(\sqrt{\frac kn}\right)\mper
\]
\end{corollary}

\subsection{Experiments with the boosting attack}

\figureref{boosting} compares the performance of the Ladder mechanism with
that of the standard Kaggle mechanism under the boosting attack. We chose
$N=12000$ as the total number of labels of which $n=4000$ labels are used for determining the
public leaderboard under either mechanism. Other parameter settings lead to a similar picture,
but these settings correspond roughly to the properties of the real data set that we will analyze later.
The Kaggle mechanism gives answers that are accurate up to a rounding error of
$10^{-5}.$ Note that $1/\sqrt{4000}\approx 0.0158$ so that the rounding error
is well below the critical level of $1/\sqrt{n}.$ The vector $y$ in the
description of our attack corresponds to the $4000$ labels used for the public
leaderboard. Since the answers given by Kaggle only depend on these labels,
the remaining labels play no role in the attack.
Importantly, the attack does
not need to know the indices of the labels used for the public leaderboard
within the entire vector of labels.

The $8000$ coordinates not used for the leaderboard remain unbiased random
bits throughout the attack as no information is revealed. In particular, the
final submission $u^*$ is completely random on those $8000$ coordinates and
only biased on the other $4000$ coordinates used for the leaderboard.
Therefore, once we evaluate the final submission $u^*$ on the test set consisting of
the remaining $8000$ coordinates, the resulting loss is close to its expected
value of $1/2,$ i.e. the expected loss of a random $0/1$-vector.
What we observe, however, is that the Kaggle mechanism
gives a strongly biased estimate of the loss of $u^*$.

The blue line in \figureref{boosting} displays the performance of the
parameter-free version of the Ladder mechanism.  Instead of
selecting all the vectors with loss at most $1/2$ we modified the attack to be
more effective against the Ladder Mechanism. Specifically, we selected all
those vectors that successfully lowered the score compared to the previous
best. As we have no information about the correlation of the remaining
vectors, there is no benefit in including them in the boosting step. 
Even with this more effective attack, the Ladder mechanism gives a result that
is correct to within the expected maximum deviation of the score on $k$ random
vectors. The intuitive reason is that every time a vector lowers the best
score seen so far, the probability of a subsequent vector crossing the
new threshold drops off by a constant factor. In particular there cannot be
more than $O(\log(k))$ such steps thus creating a bias of at
most $O(\sqrt{\log(k)/n})$ in the boosting step.

\begin{figure}[ht]
\begin{center}
\includegraphics[width=0.45\textwidth]{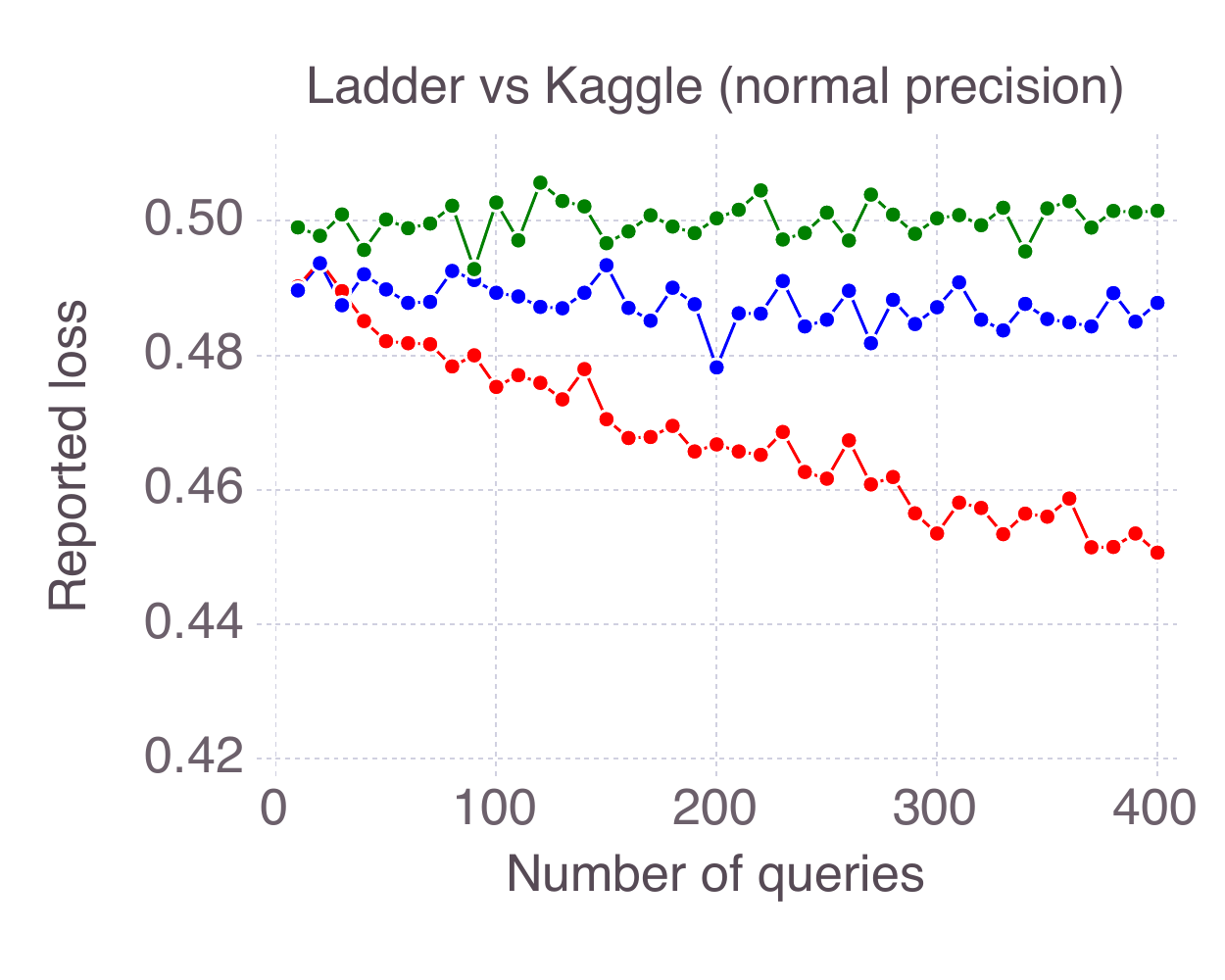}
\includegraphics[width=0.45\textwidth]{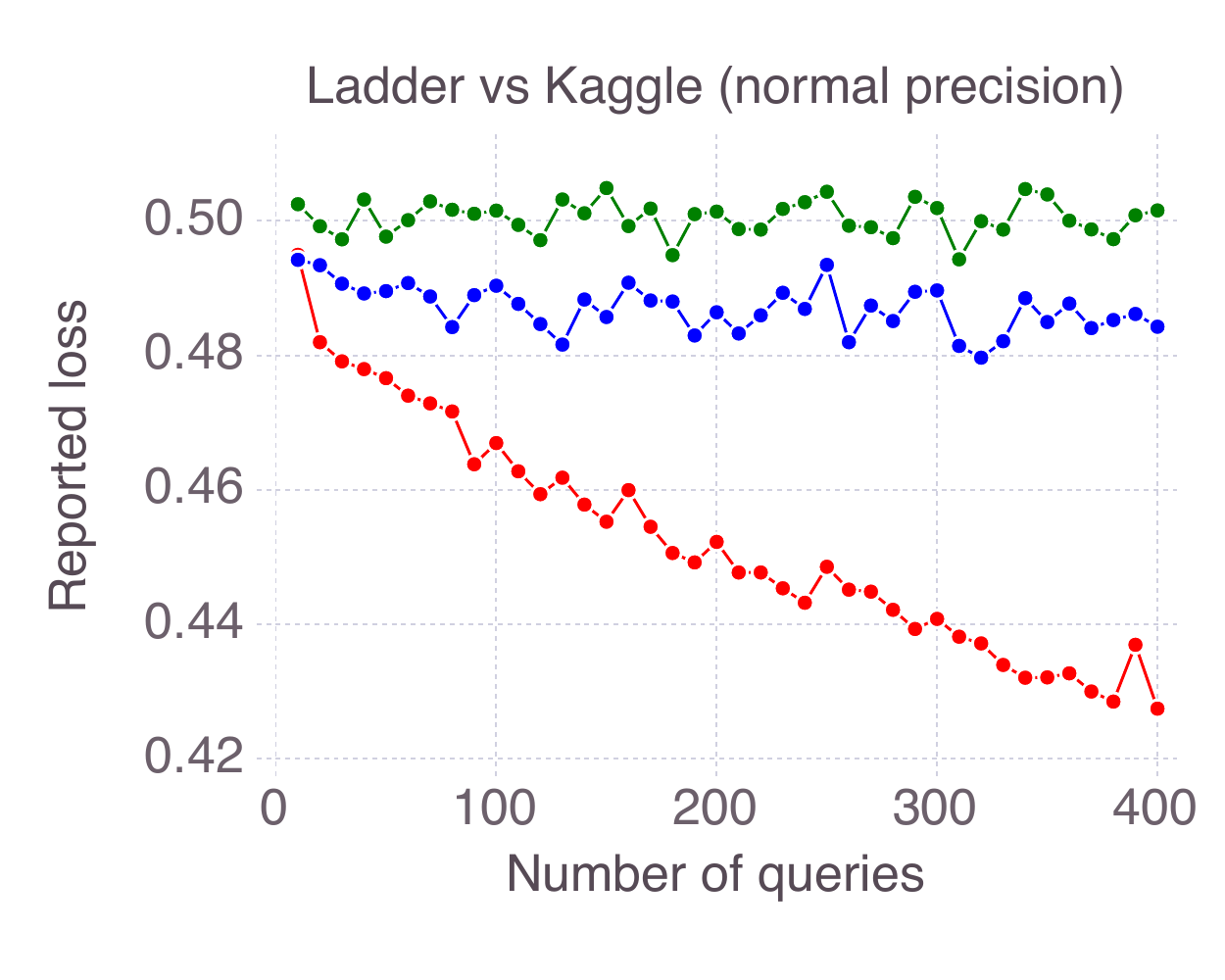}
\caption{\figurelabel{boosting} Performance of
the parameter free Ladder Mechanism compared with the Kaggle Mechanism.
Top {\bf green} line: Independent
test set. Middle {\bf blue} line: Ladder. Bottom {\bf red} line: Kaggle.
{\bf Left}: Kaggle with large rounding parameter $1/\sqrt{n}\approx 0.0158.$
{\bf Right:} Kaggle with normal rounding parameter $0.00001.$ All numbers are averaged over $5$ independent repetitions
of the experiment. Number of labels used is $n=4000.$}
\end{center}
\end{figure}

\section{Experiments on real Kaggle data}

To demonstrate the utility of the Ladder mechanism we turn to real submission
data from Kaggle's ``Photo Quality Prediction''
challenge\footnote{\url{https://www.kaggle.com/c/PhotoQualityPrediction}}.
Here is some basic information about the competition.

\begin{center}
\begin{tabular}{|l|r|}
\hline
Number of test samples&  $12000$ \\
-- used for private leaderboard&  $8400$ \\
-- used for public leaderboard&  $3600$\\
\hline
Number of submissions&  $1830$\\
-- processed successfully&  $1785$\\
\hline
Number of teams&   $200$\\
\hline
\end{tabular}
\end{center}

Our first experiment is to use the parameter-free Ladder mechanism in place of
the Kaggle mechanism across all $1785$ submissions and recompute both the
public and the private leaderboard. The resulting rankings turn out to be very close to
those computed by Kaggle. For example, \tableref{perturbations} shows the only
perturbations in the ranking among the top $10$ submissions.

\begin{table}[h]
\begin{center}
\begin{tabular}{|l|c|c|c|c|c|c|c|c|c|c|c|c|c|c|c|c|c|c|c|c|}
\hline
& \multicolumn{2}{|c|}{Private} & \multicolumn{3}{|c|}{Public} \\
\hline
Kaggle & 6 & 8 & 5 & 6 & 7\\
\hline
Ladder & 8 & 6 & 7 & 5 & 6\\
\hline
\end{tabular}
\end{center}
\caption{\tablelabel{perturbations} Perturbations in the top $10$ leaderboards}
\end{table}

\figureref{top50scores} plots the public versus private scores of the leading $50$ 
submissions (w.r.t the private leaderboard). The diagonal line indicates an
equal private and public score. The plot
a small amount of underfitting between the public and private scores.
That is, the losses on the public leaderboard generally tend to be slightly higher than
on the private leaderboard. This appears to be due random
fluctuations in the proportion of hard examples in the public holdout set. 

To assess this possibility and gain further insight into the magnitude of
statistical deviations of the scores, we randomly split the private holdout
set into two equally sized parts and recompute the leaderboards on each part.
We repeat the process $20$ times independently and look at the standard
deviations of the scores across these $20$ repetitions. \figureref{top50fresh}
shows the results demonstrating that the statistical deviations due to random
splitting are large relative to the difference in mean scores. In particular
the amount of underfitting observed on the original split is within one
standard deviation of the mean scores which cluster close to the diagonal
line. We also observed that the top $50$ scores are highly correlated so that
across different splits the points are either mostly above or mostly below the
diagonal line. This must be due to the fact that the best submissions in this
competition used related classifiers that fail to predict roughly the same
label set.

\begin{figure}[h]
\begin{center}
\includegraphics[width=0.495\textwidth]{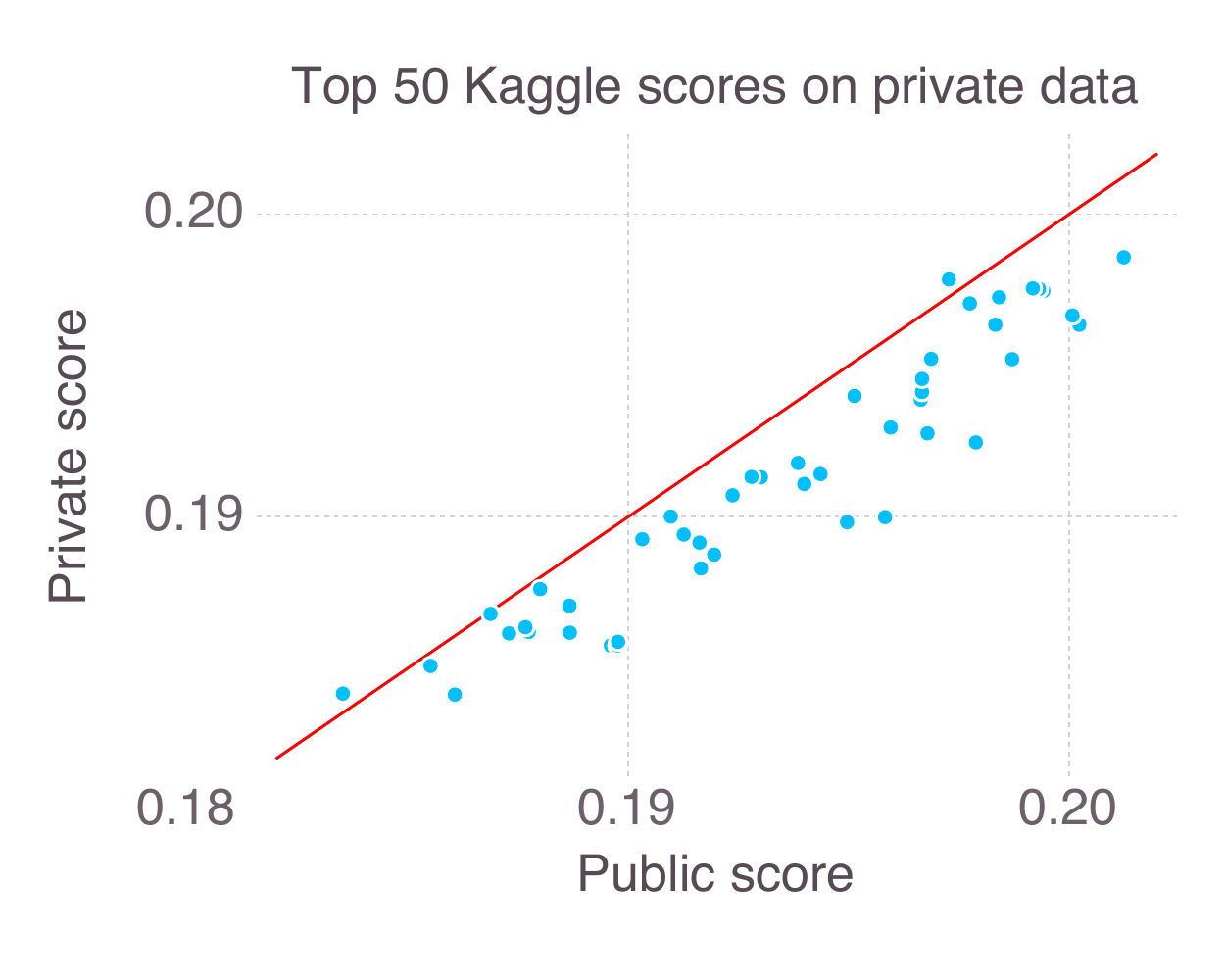}
\includegraphics[width=0.495\textwidth]{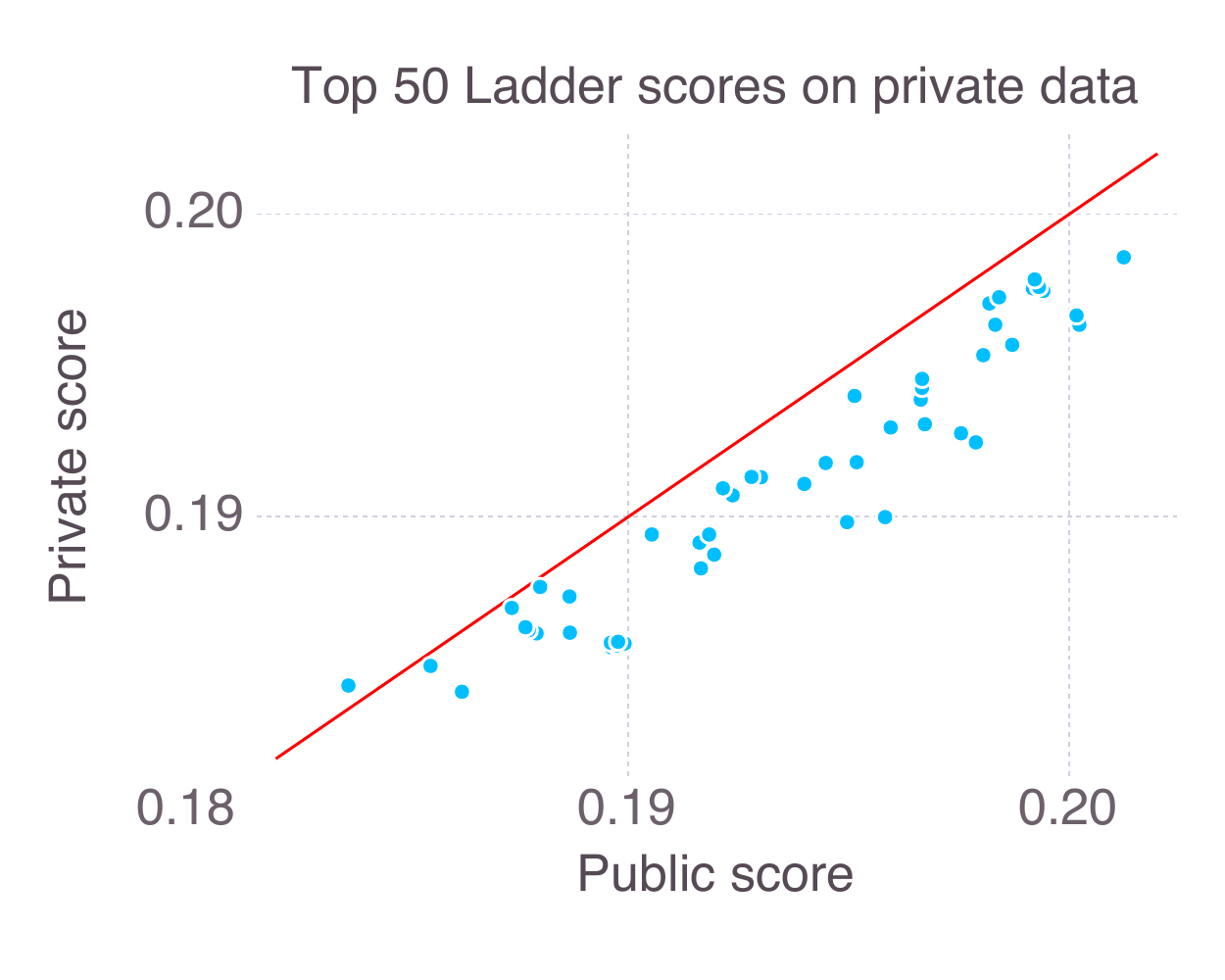}
\caption{\figurelabel{top50scores} Private versus public scores for the top
$50$ submissions. Left: Kaggle. Right: Ladder.}
\end{center}
\end{figure} 

\begin{figure}[h]
\begin{center}
\includegraphics[width=0.495\textwidth]{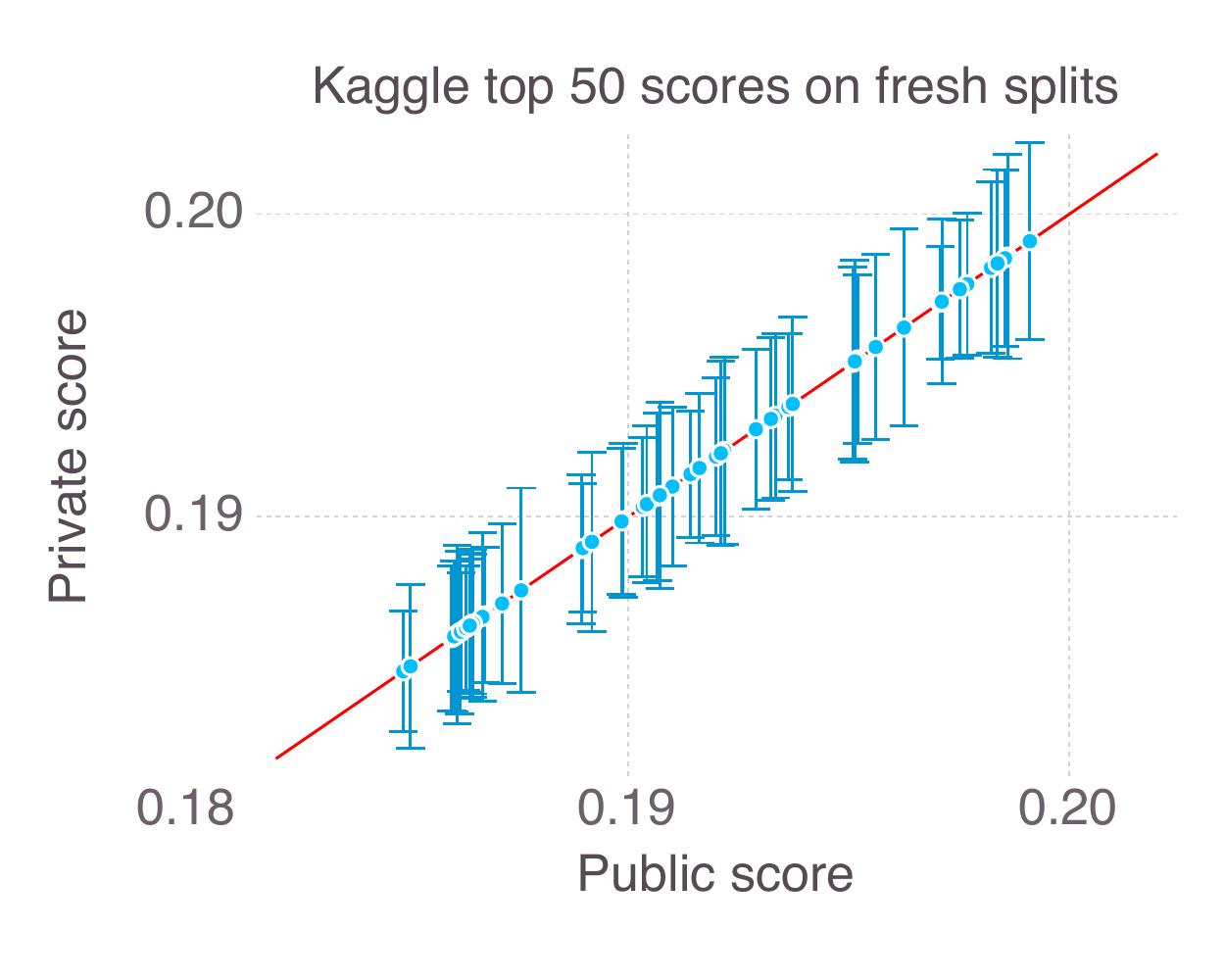}
\includegraphics[width=0.495\textwidth]{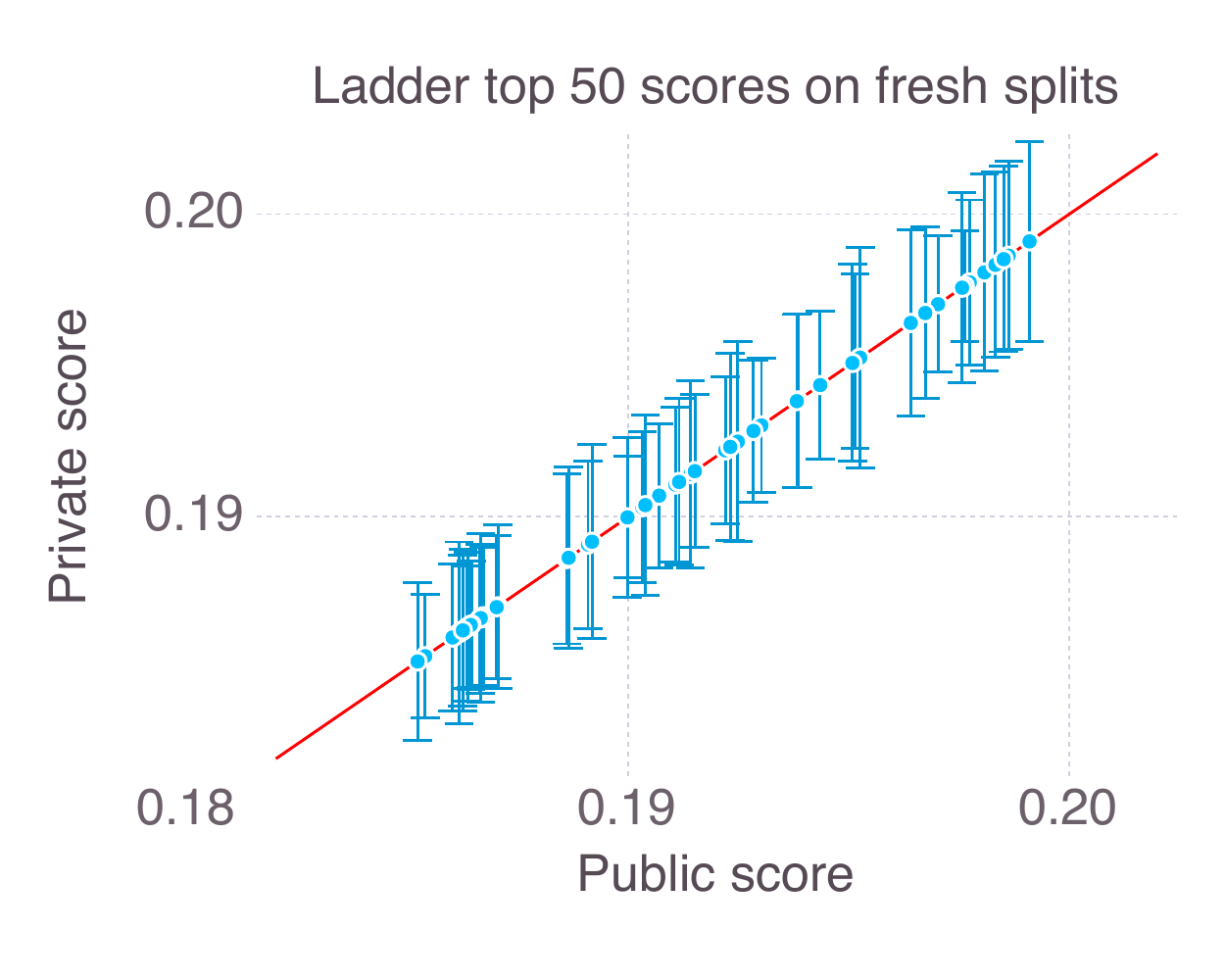}
\caption{\figurelabel{top50fresh} Fluctuations of scores across $20$
independent splits of the private data. Dots represent mean scores. Error bars
indicate a single standard deviation in each direction.}
\end{center}
\end{figure} 

\subsection{Statistical significance analysis}

To get a better sense of the statistical significance of the difference
between the scores of competing submissions we performed a sequence of
significance tests. Specifically, we considered the top $10$ submissions taken
from the Kaggle \emph{public} leaderboard and tested on the \emph{private}
data if the true score of the top submission is significantly different from
the rank~$r$ submission for $r=2,3,\dots,10.$ A suitable test of significance
is the paired $t$-test. The score of a submission is the mean of a large
number of samples in the interval $[0,2]$ and follows a sufficiently accurate
normal approximation. We chose a paired $t$-test rather than an unpaired
$t$-test, because it has far greater statistical power in our setting.  This
is primarily due to the strong correlation between competing submissions.  See
\equationref{ttest} for a definition of the test statistic. Note that the data
that determined the selection of the top $10$ classifiers is independent of
the data used to perform the significance tests.

\figureref{top10pvals} plots the resulting $p$-values before and after
correction for multiple comparisons. We see that after applying a Bonferroni
correction, the only submissions with a significantly different mean are $8$
and $9.$

\begin{figure}[h]
\begin{center}
\includegraphics[width=0.495\textwidth]{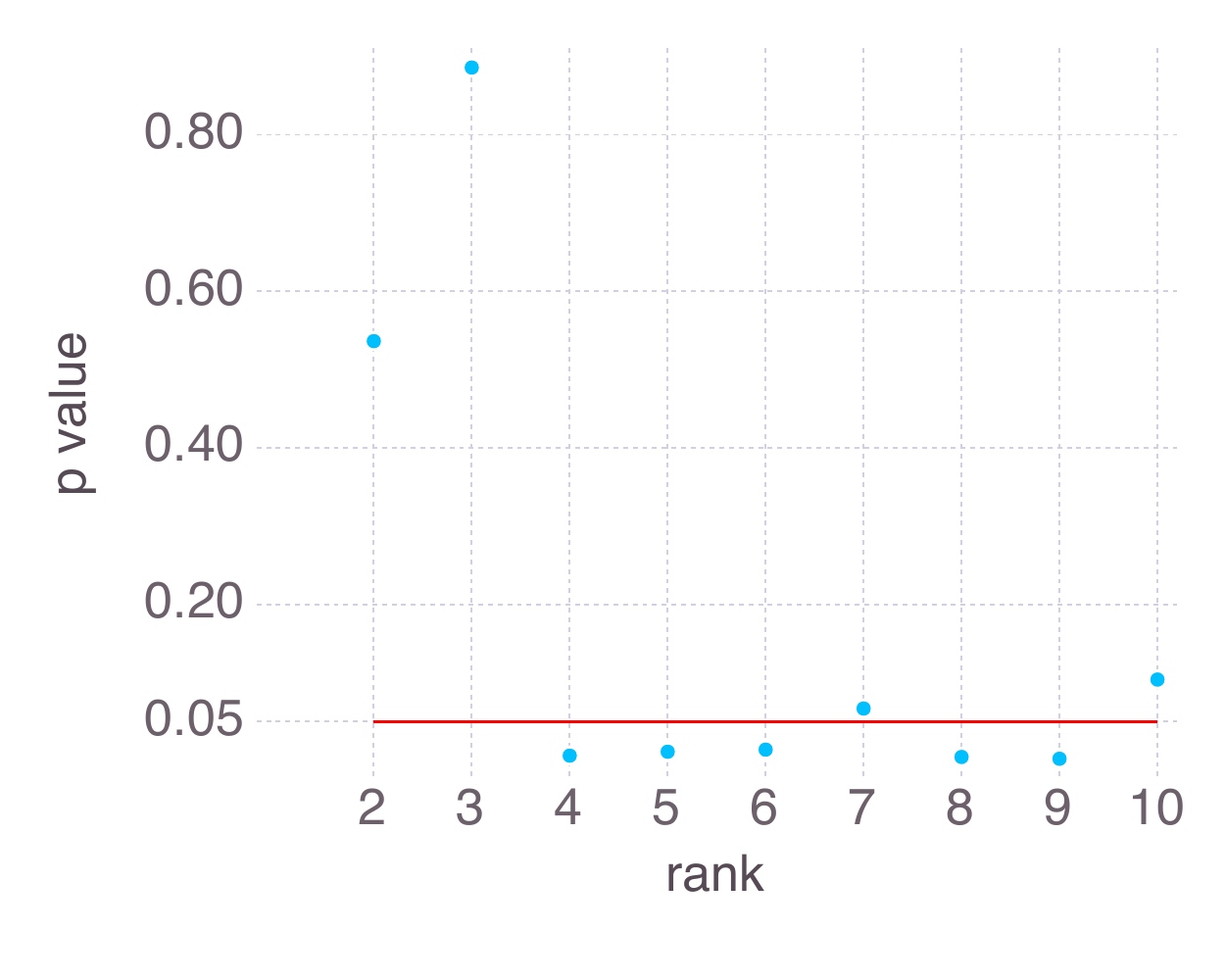}
\includegraphics[width=0.495\textwidth]{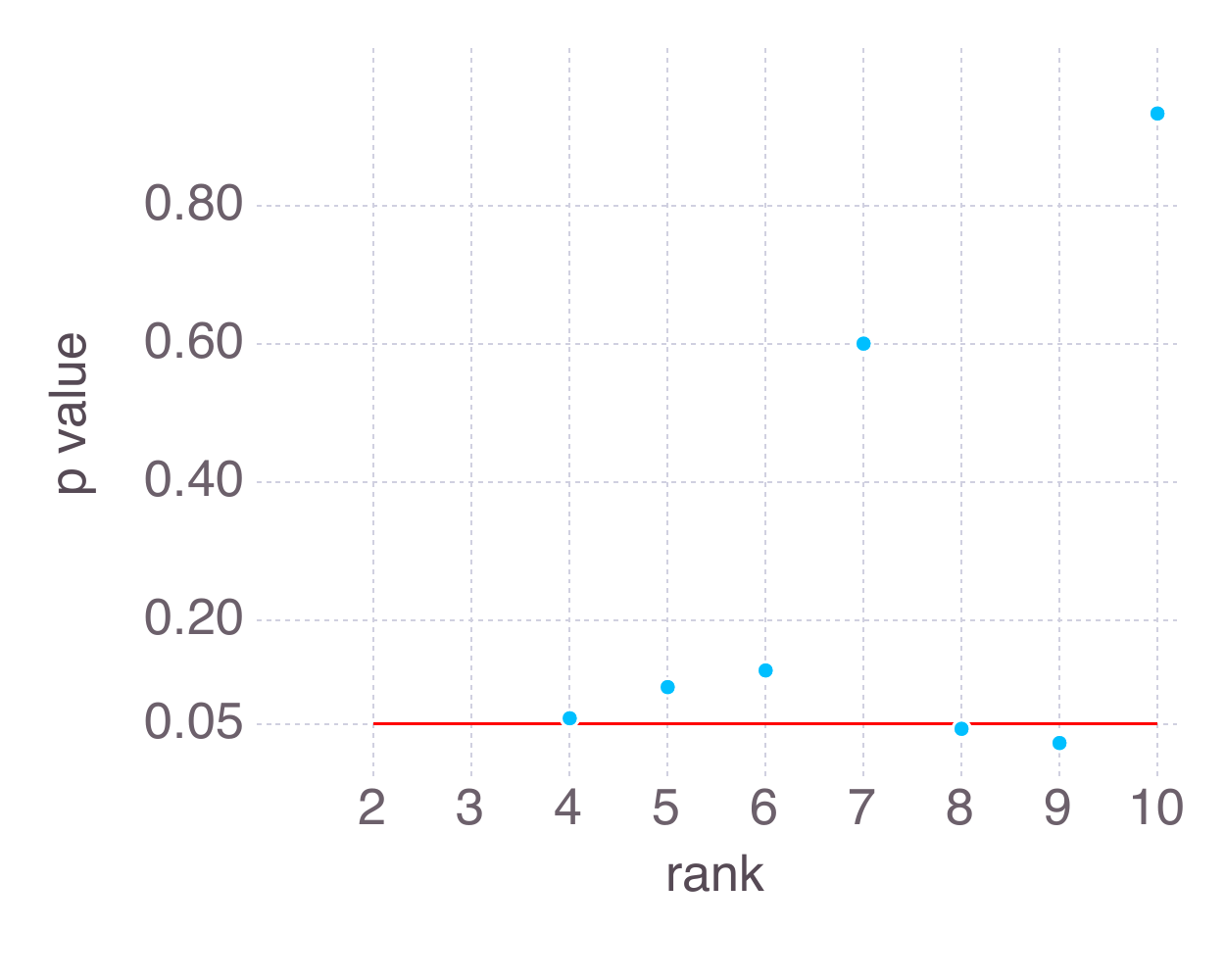}
\caption{\figurelabel{top10pvals} Significance test for the difference between
score of top submission and rank $r$ submission. Left: Before multiple
comparison correction. Right: After Bonferroni correction.}
\end{center}
\end{figure} 

These observations give further evidence that the small perturbations we saw
in the top~$10$ leaderboard between the Kaggle mechanism and the Ladder
mechanism are below the level of statistical significance.

\section{Conclusion}

We hope that the Ladder mechanism will be helpful in making machine learning
competitions more reliable. Beyond the scope of machine learning competitions, it is 
conceivable that the Ladder mechanism could be useful in other domains where
overfitting is currently a concern. For example, in the context of false
discovery in the empirical sciences~\cite{Ioannidis05,GelmanL13}, one could imagine 
using the the Ladder mechanism as a way of keeping track of scientific progress on important public data
sets.

Our algorithm can also be seen as an intuitive explanation for why overfitting
to the holdout is sometimes not a major problem even in the adaptive setting.
If indeed every analyst only uses the holdout set to test if their latest
submission is well above the previous best, then they effectively
simulate our algorithm. 

A beautiful theoretical problem is to resolve the gap between our upper and
lower bound. On the practical side, it would be interesting to use the Ladder
mechanism in a real competition. One interesting question is if the Ladder
mechanism actually encourages higher quality submissions by requiring a
certain level of statistically significant improvement over previous
submissions.

\bibliographystyle{moritz}
\bibliography{moritz}

\appendix

\section{Kaggle reference mechanism}

As we did for the Ladder Mechanism we describe the algorithm as if the analyst
was submitting classifiers $f\colon X\to Y.$ In reality the analyst only
submits a list of labels. It is easy to see that such a list of labels is
sufficient to compute the empirical loss which is all the algorithm needs to
do. The input set $S$ in the description of our algorithm corresponds to the
set of data points (and corresponding labels) that Kaggle uses for the public
leaderboard.

\begin{figure}[ht!]
\setlength{\fboxsep}{2mm}
\begin{center}
\begin{boxedminipage}{\textwidth}

\noindent {\bf Input:} Data set $S,$ rounding parameter~$\alpha>0$ (typically $0.00001$)

\noindent {\bf Algorithm:}

\begin{itm}
\item {\bf For each} round $t \leftarrow 1,2 \ldots, k:$
\begin{enum}
\item Receive function $f_t\colon X\to Y$
\item {\bf Output} $[R_S(f_t)]_\alpha.$
\end{enum}
\end{itm}
\end{boxedminipage}
\end{center}
\vspace{-3mm}
\caption{\figurelabel{kaggle} Kaggle reference mechanism. We use the notation
$[x]_\alpha$ to denote the number $x$ rounded to the nearest integer multiple of
$\alpha$.}
\end{figure}

\end{document}